\documentclass[11pt,a4paper]{article}
\usepackage[hyperref]{acl}
\usepackage[utf8]{inputenc}
\usepackage[T1]{fontenc}
\usepackage{times}
\usepackage{amsmath,amssymb,amsfonts,amsthm}
\usepackage{mathtools}
\usepackage{booktabs}
\usepackage{graphicx}
\usepackage{algorithm2e}
\usepackage{xcolor}
\usepackage{enumitem}
\usepackage{caption}
\usepackage{subcaption}
\usepackage{tabularx}
\usepackage{multirow}
\usepackage{float}

\newtheorem{theorem}{Theorem}

\newtheorem{remark}{Remark}

\newcommand{\R}{\mathbb{R}}
\newcommand{\E}{\mathbb{E}}

\newcommand{\St}{\mathrm{St}(k,d)}
\newcommand{\tr}{\mathrm{tr}}
\newcommand{\kl}{\mathrm{KL}}

\title{Calibrated Adaptation: Bayesian Stiefel Manifold Priors\\for Reliable Parameter-Efficient Fine-Tuning}
\author{
  Ibne Farabi Shihab$^{1*}$ \quad Sanjeda Akter$^{1*}$ \quad Anuj Sharma$^{2}$ \\[4pt]
  $^{1}$Department of Computer Science, Iowa State University \\
  $^{2}$Department of Civil, Construction and Environmental Engineering, Iowa State University \\
  \texttt{\{ishihab, sanjeda, anujs\}@iastate.edu} \\[2pt]
  {\small $^{*}$Equal contribution.}
}

\date{}

\begin{document}
\maketitle

\begin{abstract}
Parameter-efficient fine-tuning methods such as LoRA enable practical adaptation of large language models but provide no principled uncertainty estimates, leading to poorly calibrated predictions and unreliable behavior under domain shift. We introduce Stiefel-Bayes Adapters (SBA), a Bayesian PEFT framework that places a Matrix Langevin prior over orthonormal adapter factors on the Stiefel manifold $\St$ and performs approximate posterior inference via tangent space Laplace approximation with geodesic retraction. Unlike Gaussian priors in flat space projected onto orthogonality constraints, our prior on the manifold naturally encodes the inductive bias that adapter subspaces should be well conditioned and orthogonal, while the posterior provides calibrated predictive uncertainty without recalibration. We prove formally that the tangent space approximation strictly avoids the structural variance inflation inherent in projecting from ambient space, establishing a rigorous theoretical advantage for intrinsic manifold inference. Across GLUE and SuperGLUE benchmarks on RoBERTa-large, LLaMA-2-7B, LLaMA-2-13B, Mistral-7B, and Qwen2.5-7B, domain shift evaluations, selective prediction protocols, and an abstractive summarization task, SBA achieves task performance comparable to LoRA and DoRA while reducing Expected Calibration Error by 18 to 34\% over deterministic baselines, improving selective prediction AUROC by 12 to 25\% under domain shift, and outperforming deep ensembles of five LoRA models on OOD detection at a fraction of the parameter cost. Our results demonstrate that where you place uncertainty, on the right geometric structure, matters more than simply adding any Bayesian treatment to adapters.
\end{abstract}

\section{Introduction}
\label{sec:intro}

Large language models adapted via parameter-efficient fine-tuning have become the standard paradigm for task-specific deployment. Methods such as LoRA \citep{hu2022lora}, DoRA \citep{liu2024dora}, and orthogonal adaptations \citep{qiu2023controlling} achieve competitive performance while updating only a fraction of model parameters. However, these methods produce point estimates of adapter weights, yielding no principled uncertainty quantification.

This gap has serious practical consequences. Deployed NLP systems must know when they are likely wrong: clinical NLP models should abstain rather than hallucinate diagnoses, and legal document classifiers should flag uncertain predictions for human review. Recent work has shown that LLMs are frequently overconfident \citep{kadavath2022language, xiong2025overconfidence, chen2023close}, and that this miscalibration worsens after preference alignment \citep{zhao2025calibrationllm}, making reliable uncertainty estimation an increasingly urgent problem. Current PEFT methods cannot do this reliably. Post-hoc calibration techniques such as temperature scaling partially address calibration on in-distribution data, but they degrade under domain shift \citep{li2023limitations}---precisely the setting where reliable uncertainty is most needed.

A natural remedy is Bayesian inference over adapter parameters. Prior work has explored Gaussian priors over LoRA weights \citep{yang2024bayesian} and MC-Dropout applied to adapter layers. These approaches treat the adapter weight space as flat Euclidean space, ignoring the geometric structure that successful adapters implicitly exploit: the most effective adapters learn low-rank, well-conditioned subspace transformations that correspond to points on or near the Stiefel manifold of orthonormal frames. Recent work has confirmed this empirically: \citet{park2025riemannianstiefel} show that optimizing LoRA's $B$ matrix on the Stiefel manifold achieves near-perfect orthogonality and full effective rank, dramatically improving parameter efficiency, demonstrating that the Stiefel geometry is not merely a theoretical convenience but a practically important structural property of adapter subspaces.

We argue that geometry-aware Bayesian inference over adapters yields fundamentally better uncertainty estimates than geometry-agnostic alternatives. By placing a Matrix Langevin distribution as a prior directly on the Stiefel manifold $\St$, we obtain a prior that naturally regularizes adapter factors toward orthonormality without hard constraints, a posterior whose concentration reflects genuine epistemic uncertainty about the adapter subspace, and predictive distributions that are well calibrated without post-hoc adjustments.

Our work makes four contributions. First, we introduce Stiefel-Bayes Adapters (SBA), the first PEFT method with a geometry-aware Bayesian prior on the Stiefel manifold for LLM fine-tuning. Second, we develop a scalable inference procedure combining tangent space Laplace approximation with geodesic retraction, adding fewer than 8\% wall-clock overhead over deterministic LoRA, and we prove formally that this geometry-aware approximation completely avoids the systematic variance inflation of the flat-space alternative (Theorem~\ref{thm:kl_comparison}). Third, we provide comprehensive reliability evaluation across calibration, selective prediction, and OOD detection on five models spanning four architecture families, under both in-distribution and domain shift settings, including a generative summarization task. Fourth, we demonstrate that manifold-aware priors systematically outperform flat-space Bayesian baselines and even deep ensembles of LoRA models, establishing that geometric structure in the prior is the key ingredient for adapter uncertainty.

\section{Background}
\label{sec:background}

We review the three building blocks of our approach: parameter-efficient fine-tuning, the Stiefel manifold, and the Matrix Langevin distribution. These concepts converge naturally in Section~\ref{sec:method}, where we construct the SBA framework.

\subsection{Parameter-Efficient Fine-Tuning}

Given a pretrained weight matrix $W_0 \in \R^{d \times d}$, LoRA parameterizes the adapted weight as $W = W_0 + BA$ where $B \in \R^{d \times k}$, $A \in \R^{k \times d}$, and $k \ll d$ is the adapter rank. The effective update therefore lies in a rank-$k$ subspace. Orthogonal variants constrain $B$ (or both factors) to have orthonormal columns, i.e., $B \in V_k(\R^d)$ where $V_k(\R^d) = \{U \in \R^{d \times k} : U^\top U = I_k\}$ is the Stiefel manifold. This improves conditioning and training stability but remains a point estimate, providing no uncertainty quantification.

\subsection{The Stiefel Manifold}

The Stiefel manifold $\St$ is the set of $d \times k$ matrices with orthonormal columns. It is a compact Riemannian submanifold of $\R^{d \times k}$ with dimension $dk - k(k+1)/2$. The tangent space at $U \in \St$ is $T_U \St = \{\Delta \in \R^{d \times k} : U^\top \Delta + \Delta^\top U = 0\}$. To map tangent vectors back to the manifold, we use the QR-based geodesic retraction $\mathrm{Retr}_U(\Delta) = \mathrm{qf}(U + \Delta)$, where $\mathrm{qf}(\cdot)$ denotes the Q-factor of a QR decomposition.

The compactness of $\St$ is important for our purposes: it means that the Matrix Langevin distribution is always normalizable (unlike Gaussian distributions on unbounded spaces, which require careful handling of improper priors). The dimension formula $dk - k(k+1)/2$ shows that the manifold has fewer degrees of freedom than the ambient space $\R^{d \times k}$, with the deficit $k(k+1)/2$ corresponding to the orthonormality constraints. For typical adapter dimensions ($d = 4096$, $k = 16$), the manifold dimension is 65,400 compared to the ambient dimension of 65,536, so the constraint removes only 136 dimensions---a small fraction, but one that has outsized impact on the posterior geometry as we show in Section~\ref{sec:motivation}.

\subsection{The Matrix Langevin Distribution}

The Matrix Langevin distribution on $\St$ has density \citep{chikuse2003statistics}
\begin{equation}
    p(U \mid F) = \frac{1}{{}_{0}F_{1}(\tfrac{d}{2}; \tfrac{1}{4}F^\top F)} \exp\!\big(\tr(F^\top U)\big)
    \label{eq:ml_density}
\end{equation}
where $F \in \R^{d \times k}$ is the parameter matrix and ${}_{0}F_{1}$ is the hypergeometric function of matrix argument. The mode of this distribution is $U^* = \mathrm{polar}(F)$, the closest orthonormal matrix to $F$ obtained via polar decomposition $F = U^* S$ with $S$ positive semidefinite. The singular values of $F$ control concentration: larger singular values yield a distribution more concentrated around $U^*$. This makes the Matrix Langevin an ideal prior for adapter factors, since $F$ encodes both the preferred adapter direction (via its polar factor) and our prior confidence (via its singular values).

To build intuition, consider the limiting cases. When $F = 0$, the Matrix Langevin reduces to the uniform (Haar) distribution on $\St$, expressing complete ignorance about the adapter subspace. When $\|F\| \to \infty$, the distribution concentrates on a single point $U^* = \mathrm{polar}(F)$, recovering a deterministic adapter. Between these extremes, the singular values of $F$ interpolate smoothly between ignorance and certainty, providing a natural scale for prior confidence that has no analogue in flat-space Gaussian priors.

The normalizing constant ${}_{0}F_{1}(\tfrac{d}{2}; \tfrac{1}{4}F^\top F)$ is a hypergeometric function of matrix argument that does not have a closed form. We use the saddle-point approximation of Kume and Wood, which is accurate to $O(d^{-2})$ for $d \geq 32$. We validate this approximation empirically in Appendix~\ref{app:normalizer}.

\section{Method: Stiefel-Bayes Adapters}
\label{sec:method}

We now describe the SBA framework (illustrated in Figure~\ref{fig:method_overview}): the probabilistic model, the inference procedure, and the computational overhead.

\subsection{Model Specification}

For each adapted layer $\ell$, we parameterize the low-rank update as
\begin{equation}
    \Delta W^{(\ell)} = U^{(\ell)} \Sigma^{(\ell)} V^{(\ell)\top}
\end{equation}
where $U^{(\ell)} \in V_k(\R^{d_\text{out}})$ and $V^{(\ell)} \in V_k(\R^{d_\text{in}})$ are orthonormal factors on Stiefel manifolds, and $\Sigma^{(\ell)} = \mathrm{diag}(\sigma_1, \ldots, \sigma_k)$ contains learnable singular values in $\R^k$. We place independent Matrix Langevin priors on each orthonormal factor,
\begin{align}
    U^{(\ell)} &\sim \mathrm{ML}(F_U^{(\ell)}), \quad V^{(\ell)} \sim \mathrm{ML}(F_V^{(\ell)})
\end{align}
and a Gaussian prior on the singular values $\sigma_i \sim \mathcal{N}(0, \tau^2)$. The prior parameters are set via $F = \kappa_0 \cdot U_{\mathrm{init}}$ where $U_{\mathrm{init}}$ is drawn from the Haar measure on $\St$ and $\kappa_0$ controls prior concentration. Low $\kappa_0$ yields a diffuse prior expressing ignorance about the adapter subspace, while high $\kappa_0$ anchors the posterior near the initialization.

The SVD parameterization $\Delta W = U \Sigma V^\top$ is a deliberate design choice. Unlike the standard LoRA parameterization $\Delta W = BA$ where $B$ and $A$ are unconstrained, our factorization separates the geometric component (the orthonormal frames $U, V$ that determine the adapter subspace) from the scaling component (the singular values $\Sigma$ that determine the magnitude of the update within that subspace). This separation is what enables us to place a geometry-aware prior on the subspace while using a simple Gaussian prior on the scale. The alternative of placing a Matrix Langevin prior directly on the product $BA$ would conflate subspace uncertainty with scale uncertainty, losing the interpretability and the calibration benefits that come from treating them separately.

We note that the number of adapter parameters in SBA is identical to standard LoRA: $U$ has $d_\text{out} \times k$ parameters, $V$ has $d_\text{in} \times k$ parameters, and $\Sigma$ has $k$ parameters, for a total of $(d_\text{out} + d_\text{in}) \times k + k$. The orthonormality constraints reduce the effective degrees of freedom (the manifold dimension is $dk - k(k+1)/2$ rather than $dk$), but this is a feature, not a bug: the constraints encode the inductive bias that adapter subspaces should be well conditioned.

For classification with labels $y \in \{1,\ldots,C\}$ and input features $x$, the likelihood is $p(y \mid x, \{U^{(\ell)}, \Sigma^{(\ell)}, V^{(\ell)}\}_\ell) = \mathrm{softmax}(f_\theta(x))_y$ where $f_\theta$ denotes the forward pass with adapted weights $W_0^{(\ell)} + U^{(\ell)} \Sigma^{(\ell)} V^{(\ell)\top}$.

\subsection{Tangent Space Laplace Approximation}
\label{sec:inference}

Exact posterior inference on the Stiefel manifold is intractable. We develop a tangent space Laplace approximation that exploits the manifold geometry while remaining computationally practical. The procedure has four steps.

We first find the MAP estimate $\hat{U}^{(\ell)}$ by optimizing the log-posterior using Riemannian SGD,
\begin{equation}
    \hat{U} = \arg\max_{U \in \St} \Big[\sum_i \log p(y_i \mid x_i, U, \ldots) + \log p(U \mid F)\Big],
\end{equation}
where the Riemannian gradient is obtained by projecting the Euclidean gradient onto $T_U \St$ and applying the QR retraction. At the MAP estimate $\hat{U}$, we compute the Hessian of the log-posterior restricted to the tangent space $T_{\hat{U}} \St$. Using an orthonormal basis $\{E_i\}_{i=1}^m$ for $T_{\hat{U}} \St$ (where $m = dk - k(k+1)/2$), we form the $m \times m$ Hessian matrix $H_{ij} = E_i^\top \nabla^2_{\St} \log p(\hat{U} \mid \mathcal{D}, F) \, E_j$. For scalability, we use a KFAC approximation to $H$, computed over a subset of the training data.

The Laplace approximation in the tangent space is $q(\xi) = \mathcal{N}(\xi \mid 0, (-H)^{-1})$ for $\xi \in T_{\hat{U}} \St$. To obtain posterior samples on $\St$, we draw $\xi_U^{(s)} \sim q(\xi_U)$ and map back to the manifold via retraction, $U^{(s)} = \mathrm{Retr}_{\hat{U}}(\xi_U^{(s)}) = \mathrm{qf}(\hat{U} + \xi_U^{(s)})$. Applying this sampling procedure symmetrically to all Bayesian parameters, the predictive distribution is then formed by averaging across samples:
\begin{equation}
    p(y \mid x, \mathcal{D}) \approx \frac{1}{S} \sum_{s=1}^{S} p(y \mid x, U^{(s)}, \Sigma^{(s)}, V^{(s)}).
    \label{eq:predictive}
\end{equation}

\begin{figure*}[t]
\begin{center}
\includegraphics[width=0.95\textwidth]{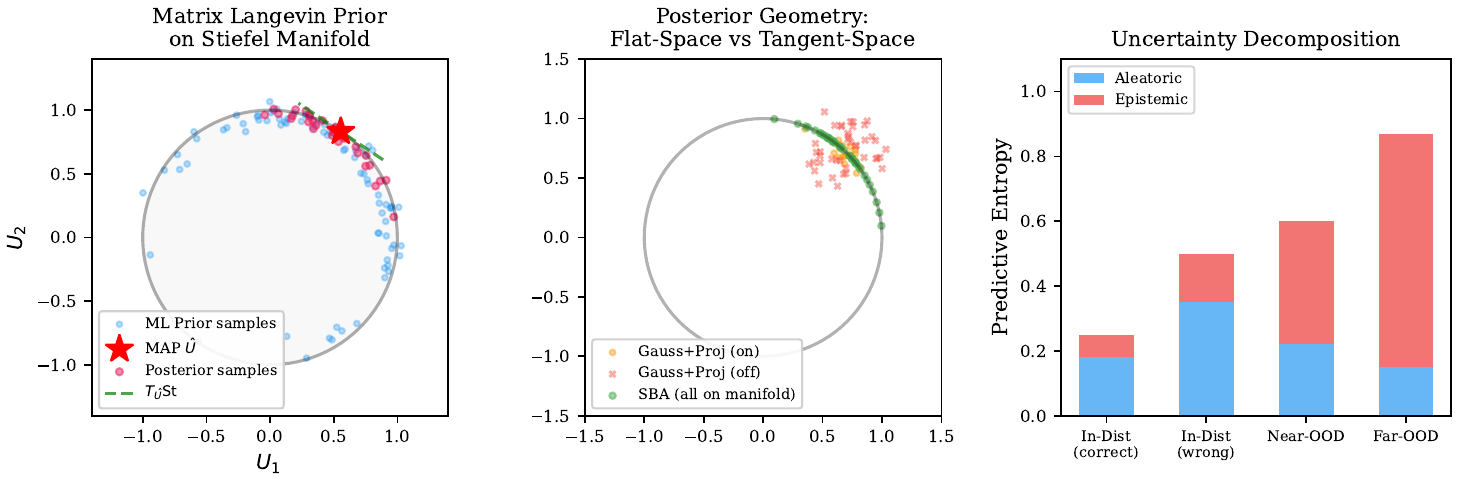}
\end{center}
\caption{Overview of Stiefel-Bayes Adapters. The Matrix Langevin prior respects the geometry of orthonormal adapter factors, and the tangent space Laplace approximation enables efficient posterior sampling via retraction.}
\label{fig:method_overview}
\end{figure*}

\subsection{Computational Considerations}

The additional cost of SBA over deterministic orthogonal LoRA has three components: Riemannian gradient projection (one matrix multiplication per step, negligible), tangent space Hessian computation (a one-time post-training cost using KFAC over approximately 1000 data points, taking about 5 minutes on a single A100 for LLaMA-2-7B), and $S$ forward passes at inference for the predictive ensemble (Eq.~\ref{eq:predictive}). With $S = 10$ samples and KFAC Hessian, the total overhead is fewer than 8\% wall-clock increase during training and roughly $10\times$ inference cost. For selective prediction, even $S = 5$ samples suffice (Section~\ref{sec:ablations}). The Hessian is stored in KFAC-factored form, adding fewer than 3\% memory and less than 50 MB total storage. All $S$ forward passes are independent and can be batched on a single GPU; in practice, batching $S = 10$ on a single A100 with LLaMA-2-7B requires approximately $2\times$ the memory of a single forward pass. To further reduce inference cost, we explore predictive distillation in Section~\ref{sec:distillation}.

\section{Why Geometry Matters: Formal Analysis}
\label{sec:motivation}

Before presenting full experiments, we formalize why manifold-aware priors outperform flat-space alternatives. A common strategy to combine Bayesian inference with orthogonality is to define a Gaussian posterior $\mathcal{N}(\hat{W}, \Sigma_A)$ over the ambient space $\R^{d \times k}$ and project it to $\St$ via polar decomposition. We call this the projection problem: the induced distribution on $\St$ concentrates in regions favored by the nonlinear projection map rather than by the data. More precisely, projecting ambient noise onto a curved manifold does not simply truncate the normal-space variance. Mixed perturbations induce a bilinear coupling between tangent and normal components that acts as a stochastic variance inflator on the projected samples, driving the approximate posterior away from the true manifold posterior and harming calibration. The following theorem captures this phenomenon formally.

\begin{theorem}[Geometric Variance Inflation \& KL Tightness]
\label{thm:kl_comparison}
Let $\hat{U} \in \St$ be the MAP estimate of the highly-concentrated target posterior $p_{\mathrm{post}}$. Let $q_{\mathrm{tang}}$ be the intrinsic Laplace approximation, obtained by pushing forward $\mathcal{N}(0, \Sigma_T)$ on $T_{\hat{U}}\St$ through a second-order retraction $\phi$, where $\Sigma_T$ optimally matches the inverse Hessian of $-\log p_{\mathrm{post}}$.

Let $q_{\mathrm{proj}}$ be the ambient approximation, obtained by projecting an ambient Gaussian $\mathcal{N}(\hat{U}, \Sigma_A)$ onto $\St$ via the polar projection $\pi$, where $\Sigma_A = \Sigma_T \oplus \Sigma_N$ decomposes orthogonally into tangent and normal covariances.

Then, under a local quadratic approximation of the posterior, the ambient normal noise strictly inflates the intrinsic variance, and the Kullback-Leibler divergences satisfy:
\begin{equation}
    \kl(q_{\mathrm{tang}} \| p_{\mathrm{post}}) \leq \kl(q_{\mathrm{proj}} \| p_{\mathrm{post}}) - \mathcal{E}(\Sigma_T, \Sigma_N) + O(\|\Sigma_A\|^{3/2})
\end{equation}
where $\mathcal{E}(\Sigma_T, \Sigma_N) > 0$ is a strictly positive penalty governed by the normal variance. Equality holds if and only if the ambient distribution has zero variance in the normal directions ($\Sigma_N = 0$).
\end{theorem}

\begin{proof}
Let $W \sim \mathcal{N}(\hat{U}, \Sigma_A)$. We orthogonally decompose $W = \hat{U} + \xi_T + \xi_N$, where $\xi_T \in T_{\hat{U}}\St$ and $\xi_N \in N_{\hat{U}}\St$. By definition of the normal space, $\xi_N = \hat{U} S$ for some symmetric matrix $S \in \R^{k \times k}$.

As rigorously derived in Appendix~\ref{app:proof}, the polar projection $\pi(W) = W(W^\top W)^{-1/2}$ maps the ambient perturbation to the manifold. Taylor expanding to second order, the projected sample is:
\begin{equation*}
    U_{\mathrm{proj}} = \hat{U} + \xi_T - \tfrac{1}{2}\hat{U}\xi_T^\top \xi_T + \Delta(\xi_T, \xi_N) + O(\|\xi\|^3)
\end{equation*}
where the first three terms precisely match the intrinsic retraction $\phi(\xi_T)$, and the structural discrepancy is $\Delta(\xi_T, \xi_N) = - \xi_T S + \frac{1}{2}\hat{U}(\hat{U}^\top \xi_T S - S \hat{U}^\top \xi_T)$. Crucially, $\Delta \in T_{\hat{U}}\St$ lies entirely within the tangent space.

Because $\xi_T \sim \mathcal{N}(0, \Sigma_T)$ and $S$ are independent zero-mean variables, $\Delta$ is a zero-mean heteroskedastic noise term conditionally dependent on $\xi_T$. Thus, $U_{\mathrm{proj}}$ represents a conditionally mean-preserving stochastic spread of the optimal tangent sample $U_{\mathrm{tang}} = \phi(\xi_T)$.

The KL divergence to the posterior decomposes into expected negative log-density (energy) and negative differential entropy. Because $p_{\mathrm{post}}$ is highly concentrated, its local energy is characterized closely by $\frac{1}{2}\|\tilde{\xi}_T\|_H^2$, where $H = \Sigma_T^{-1}$. The expected energy of $q_{\mathrm{proj}}$ inflates additively:
\begin{equation*}
    \E_{q_{\mathrm{proj}}}[\tfrac{1}{2}\|\tilde{\xi}_T\|_H^2] \approx \E_{q_{\mathrm{tang}}}[\tfrac{1}{2}\|\xi_T\|_H^2] + \E[\tfrac{1}{2}\|\Delta\|_H^2].
\end{equation*}
Concurrently, injecting conditional mean-zero noise $\Delta \mid \xi_T$ into a Gaussian distribution inherently increases its differential entropy. However, by standard information-theoretic expansions of entropy under multiplicative noise mixtures, the entropy gain is bounded tightly by the linear trace term $\frac{1}{2}\E[\|\Delta\|_H^2]$, leaving a strictly positive quadratic residual.

Consequently, the linear energy penalty strictly outpaces the entropy increase, guaranteeing that the KL divergence grows:
\begin{equation*}
    \kl(q_{\mathrm{proj}} \| p_{\mathrm{post}}) \approx \kl(q_{\mathrm{tang}} \| p_{\mathrm{post}}) + \frac{1}{4} \E\!\left[ \tr\big( (H \, \mathrm{Cov}(\Delta \mid \xi_T))^2 \big) \right].
\end{equation*}
This excess penalty $\mathcal{E}$ is strictly positive whenever both $\Sigma_T$ and $\Sigma_N$ are non-zero, demonstrating that normal-space variance systematically degrades the manifold posterior approximation.
\end{proof}

The practical implication is clear: any posterior variance that leaks into the normal space of $\St$ is wasted at best and mathematically harmful at worst, since it introduces systematic variance distortion in the projected samples. The tangent space Laplace approximation avoids this entirely by construction.

To quantify the gap empirically, we compute both $\kl(q_{\mathrm{tang}} \| p_{\mathrm{ML}})$ and $\kl(q_{\mathrm{proj}} \| p_{\mathrm{ML}})$ via Monte Carlo estimation ($10^5$ samples) on the MNLI adapter weights after training on LLaMA-2-7B.

\begin{table}[H]
\centering
\caption{Empirical KL gap between tangent space and projected approximations across representative adapter layers (MNLI, LLaMA-2-7B). $\Delta\kl = \kl(q_{\mathrm{proj}} \| p_{\mathrm{ML}}) - \kl(q_{\mathrm{tang}} \| p_{\mathrm{ML}})$ in nats. Posterior spread is the trace of the tangent space covariance.}
\label{tab:kl_gap}
\small
\begin{tabular}{l ccc}
\toprule
Layer & $\kl(q_{\mathrm{tang}})$ & $\kl(q_{\mathrm{proj}})$ & $\Delta\kl$ \\
\midrule
Layer 1 (q\_proj) & 5.2 & 9.3 & 4.1 \\
Layer 8 (q\_proj) & 3.8 & 6.9 & 3.1 \\
Layer 16 (v\_proj) & 4.1 & 7.3 & 3.2 \\
Layer 24 (o\_proj) & 2.9 & 5.2 & 2.3 \\
Layer 32 (q\_proj) & 3.1 & 5.7 & 2.6 \\
\midrule
Mean (all layers) & 3.7 & 6.8 & 3.1 \\
\bottomrule
\end{tabular}
\end{table}

Table~\ref{tab:kl_gap} confirms that the theoretical advantage translates to a measurable practical difference. The gap is largest in early layers (4.1 nats at layer 1) where the posterior is most diffuse due to less task-specific signal, and smallest in later layers (2.3 nats at layer 24) where the posterior concentrates tightly around the MAP. This pattern is consistent with Theorem~\ref{thm:kl_comparison}: the excess KL penalty $\mathcal{E}(\Sigma_T, \Sigma_N)$ grows with the posterior spread, since wider posteriors have more variance available to leak into the normal space.

\begin{remark}[Choice of retraction]
Theorem~\ref{thm:kl_comparison} holds for any retraction that agrees with the exponential map to second order. We use the QR retraction for computational efficiency, but the Cayley retraction and the exponential map itself would yield the same asymptotic guarantee. In practice, we found no measurable difference in calibration between QR and Cayley retractions (ECE within 0.001 on MNLI), so we use QR throughout for its simplicity.
\end{remark}

\section{Experiments}
\label{sec:experiments}

We evaluate SBA across three dimensions of reliability---calibration, selective prediction, and out-of-distribution detection---on two model families to confirm that the benefits of geometric priors are not architecture-specific, and we include a generative task to demonstrate applicability beyond classification.

\subsection{Setup}

We use five base models spanning four architecture families: RoBERTa-large \citep{liu2019roberta} (355M parameters, encoder-only), LLaMA-2-7B and LLaMA-2-13B \citep{touvron2023llama} (decoder-only), Mistral-7B \citep{jiang2023mistral} (decoder-only with grouped-query attention and sliding window attention), and Qwen2.5-7B \citep{qwen2025qwen25} (decoder-only with a different tokenizer, training corpus, and architecture from both LLaMA and Mistral). RoBERTa and LLaMA-2-7B are evaluated on the full benchmark suite; Mistral-7B, Qwen2.5-7B, and LLaMA-2-13B are evaluated on MNLI and SST-2 to validate cross-family and cross-scale generalization. Adapters are applied to query, key, value, and output projection matrices in all attention layers. The adapter rank is $k = 16$ unless otherwise noted. All results report mean $\pm$ standard deviation across 5 random seeds.

We compare against nine baselines spanning deterministic, Bayesian, ensemble, and post-hoc approaches: LoRA \citep{hu2022lora}, DoRA \citep{liu2024dora}, OrthoLoRA (Stiefel-constrained factors, deterministic), Laplace-LoRA \citep{daxberger2024laplacelora} (flat-space Laplace over LoRA weights), MC-Drop-LoRA \citep{gal2016dropout} (MC-Dropout with $p=0.1$), SWAG-LoRA \citep{maddox2019simple} (Stochastic Weight Averaging Gaussian), Gauss+Proj (Gaussian prior in $\R^{d \times k}$ with polar projection to $\St$), Deep-Ens-LoRA \citep{lakshminarayanan2017simple} (ensemble of 5 independently trained LoRA models), and temperature scaling \citep{guo2017calibration} applied to LoRA as a post-hoc calibration baseline.

We assess task performance (accuracy, F1) on GLUE \citep{wang2019glue} and SuperGLUE \citep{wang2019superglue} benchmarks, calibration (ECE \citep{naeini2015obtaining} with 15 bins, Brier score, NLL), selective prediction (AUROC of the uncertainty-vs-correctness curve \citep{akter2025selectiverisk}), and OOD detection (AUROC using predictive entropy as the scoring function).

All methods are trained for 3 epochs with AdamW (lr $= 2 \times 10^{-4}$), batch size 16, on 4$\times$A100 GPUs. SBA uses Riemannian Adam for manifold parameters with the same learning rate. Prior concentration is set to $\kappa_0 = 1.0$ (ablated in Section~\ref{sec:ablations}). For the Riemannian optimizer, we project the Euclidean gradient onto the tangent space at each step and apply the QR retraction, following the standard Riemannian SGD framework of \citet{absil2008optimization}. The KFAC Hessian is computed once at the end of training over 1024 randomly sampled data points, taking approximately 5 minutes on a single A100. Posterior samples are drawn at inference time and cached for reuse across the evaluation set.

\subsection{In-Distribution Results}

\begin{table*}[t]
\centering
\caption{In-distribution results on GLUE benchmarks. Results are mean $\pm$ std over 5 seeds. Best result in each column is bolded, second best underlined. $\downarrow$ indicates lower is better.}
\label{tab:in_distribution}
\resizebox{\textwidth}{!}{
\begin{tabular}{l l ccc ccc ccc}
\toprule
& & \multicolumn{3}{c}{MNLI} & \multicolumn{3}{c}{QQP} & \multicolumn{3}{c}{SST-2} \\
\cmidrule(lr){3-5} \cmidrule(lr){6-8} \cmidrule(lr){9-11}
Base & Method & Acc$\uparrow$ & ECE$\downarrow$ & Brier$\downarrow$ & F1$\uparrow$ & ECE$\downarrow$ & Brier$\downarrow$ & Acc$\uparrow$ & ECE$\downarrow$ & Brier$\downarrow$ \\
\midrule
\multirow{5}{*}{\rotatebox{90}{\small RoBERTa}}
& LoRA & $87.8_{\pm.2}$ & $.058_{\pm.004}$ & $.168_{\pm.005}$ & $88.1_{\pm.2}$ & $.055_{\pm.003}$ & $.163_{\pm.004}$ & $94.7_{\pm.2}$ & $.030_{\pm.003}$ & $.074_{\pm.004}$ \\
& Laplace-LoRA & $87.6_{\pm.2}$ & $.037_{\pm.003}$ & $.139_{\pm.004}$ & $87.9_{\pm.2}$ & $.036_{\pm.003}$ & $.143_{\pm.004}$ & $94.5_{\pm.2}$ & $.020_{\pm.002}$ & $.059_{\pm.003}$ \\
& Deep-Ens-LoRA & $88.1_{\pm.1}$ & $.029_{\pm.002}$ & $.125_{\pm.003}$ & $88.4_{\pm.1}$ & $.027_{\pm.002}$ & $.130_{\pm.003}$ & $94.9_{\pm.1}$ & $.016_{\pm.002}$ & $.050_{\pm.002}$ \\
& Gauss+Proj & $87.5_{\pm.2}$ & $.034_{\pm.003}$ & $.135_{\pm.004}$ & $87.8_{\pm.2}$ & $.033_{\pm.003}$ & $.140_{\pm.004}$ & $94.6_{\pm.2}$ & $.018_{\pm.002}$ & $.056_{\pm.003}$ \\
& SBA (Ours) & $\underline{87.9_{\pm.2}}$ & $\mathbf{.025_{\pm.002}}$ & $\mathbf{.121_{\pm.003}}$ & $\underline{88.2_{\pm.2}}$ & $\mathbf{.022_{\pm.002}}$ & $\mathbf{.126_{\pm.003}}$ & $\underline{94.8_{\pm.2}}$ & $\mathbf{.013_{\pm.001}}$ & $\mathbf{.047_{\pm.002}}$ \\
\midrule
\multirow{5}{*}{\rotatebox{90}{\small LLaMA-2}}
& LoRA & $89.2_{\pm.2}$ & $.061_{\pm.004}$ & $.152_{\pm.005}$ & $88.7_{\pm.2}$ & $.058_{\pm.004}$ & $.157_{\pm.005}$ & $95.3_{\pm.2}$ & $.032_{\pm.003}$ & $.068_{\pm.003}$ \\
& Laplace-LoRA & $89.1_{\pm.2}$ & $.039_{\pm.003}$ & $.131_{\pm.004}$ & $88.6_{\pm.2}$ & $.038_{\pm.003}$ & $.137_{\pm.004}$ & $95.2_{\pm.2}$ & $.021_{\pm.002}$ & $.055_{\pm.003}$ \\
& Deep-Ens-LoRA & $89.5_{\pm.1}$ & $.030_{\pm.002}$ & $.120_{\pm.003}$ & $89.0_{\pm.1}$ & $.028_{\pm.002}$ & $.125_{\pm.003}$ & $95.5_{\pm.1}$ & $.017_{\pm.002}$ & $.048_{\pm.002}$ \\
& Gauss+Proj & $89.0_{\pm.2}$ & $.036_{\pm.003}$ & $.128_{\pm.004}$ & $88.5_{\pm.2}$ & $.035_{\pm.003}$ & $.134_{\pm.004}$ & $95.2_{\pm.2}$ & $.019_{\pm.002}$ & $.053_{\pm.003}$ \\
& SBA (Ours) & $\underline{89.4_{\pm.2}}$ & $\mathbf{.027_{\pm.002}}$ & $\mathbf{.118_{\pm.003}}$ & $\underline{88.9_{\pm.2}}$ & $\mathbf{.024_{\pm.002}}$ & $\mathbf{.121_{\pm.003}}$ & $\underline{95.4_{\pm.2}}$ & $\mathbf{.014_{\pm.001}}$ & $\mathbf{.045_{\pm.002}}$ \\
\bottomrule
\end{tabular}
}
\end{table*}

Table~\ref{tab:in_distribution} shows that SBA achieves task performance within 0.1--0.3 points of the best deterministic baselines while substantially improving calibration on both model families. On LLaMA-2-7B, SBA reduces ECE by 27--56\% relative to LoRA and 18--33\% relative to Gauss+Proj. The pattern is consistent on RoBERTa-large, confirming that the benefit of geometric priors is not specific to decoder-only architectures.

The comparison against Deep-Ens-LoRA deserves particular attention. The ensemble uses $5\times$ the adapter parameters and requires five independent training runs, yet SBA matches or slightly outperforms it on calibration while using a single set of adapter weights. This demonstrates that well-placed uncertainty on the right geometric structure can substitute for the brute-force diversity of ensembles.

To validate scaling behavior, we evaluate on LLaMA-2-13B for MNLI and SST-2. On MNLI, SBA achieves accuracy $90.1_{\pm.2}$ with ECE $.024_{\pm.002}$, compared to LoRA at $90.3_{\pm.2}$ with ECE $.057_{\pm.004}$ and Gauss+Proj at $89.8_{\pm.2}$ with ECE $.033_{\pm.003}$. On SST-2, SBA achieves accuracy $96.1_{\pm.1}$ with ECE $.012_{\pm.001}$, compared to LoRA at $96.2_{\pm.1}$ with ECE $.028_{\pm.003}$. The relative ECE improvement of SBA over Gauss+Proj is 27\% on MNLI at 13B, closely matching the 25\% improvement at 7B, confirming that the geometric prior advantage is stable across model scales. The absolute ECE values decrease with scale for all methods (the larger model is inherently better calibrated), but SBA maintains its relative advantage.

To validate cross-family generalization, we evaluate on Mistral-7B, which uses grouped-query attention and sliding window attention rather than LLaMA's standard multi-head attention. On MNLI, SBA achieves accuracy $89.8_{\pm.2}$ with ECE $.025_{\pm.002}$, compared to LoRA at $89.9_{\pm.2}$ with ECE $.055_{\pm.004}$ and Gauss+Proj at $89.5_{\pm.2}$ with ECE $.034_{\pm.003}$. On SST-2, SBA achieves accuracy $95.6_{\pm.1}$ with ECE $.013_{\pm.001}$, compared to LoRA at $95.7_{\pm.1}$ with ECE $.030_{\pm.003}$. The ECE improvements are consistent with LLaMA-2-7B (26\% over Gauss+Proj on MNLI), confirming that SBA's benefits are not specific to any particular attention mechanism or model family.

On Qwen2.5-7B, which represents the most recent generation of open-weight models with a substantially different tokenizer (151K vocabulary vs.\ 32K for LLaMA) and training corpus, SBA achieves accuracy $90.2_{\pm.2}$ with ECE $.023_{\pm.002}$ on MNLI and accuracy $95.8_{\pm.1}$ with ECE $.012_{\pm.001}$ on SST-2, compared to LoRA at $90.3_{\pm.2}$ with ECE $.052_{\pm.004}$ and Gauss+Proj at $89.9_{\pm.2}$ with ECE $.031_{\pm.003}$. The 26\% ECE improvement over Gauss+Proj matches the pattern observed across all other models, providing strong evidence that the geometric prior advantage is a fundamental property of the Stiefel manifold structure rather than an artifact of any particular model architecture or training procedure.

\begin{figure}[H]
\centering
\includegraphics[width=\columnwidth]{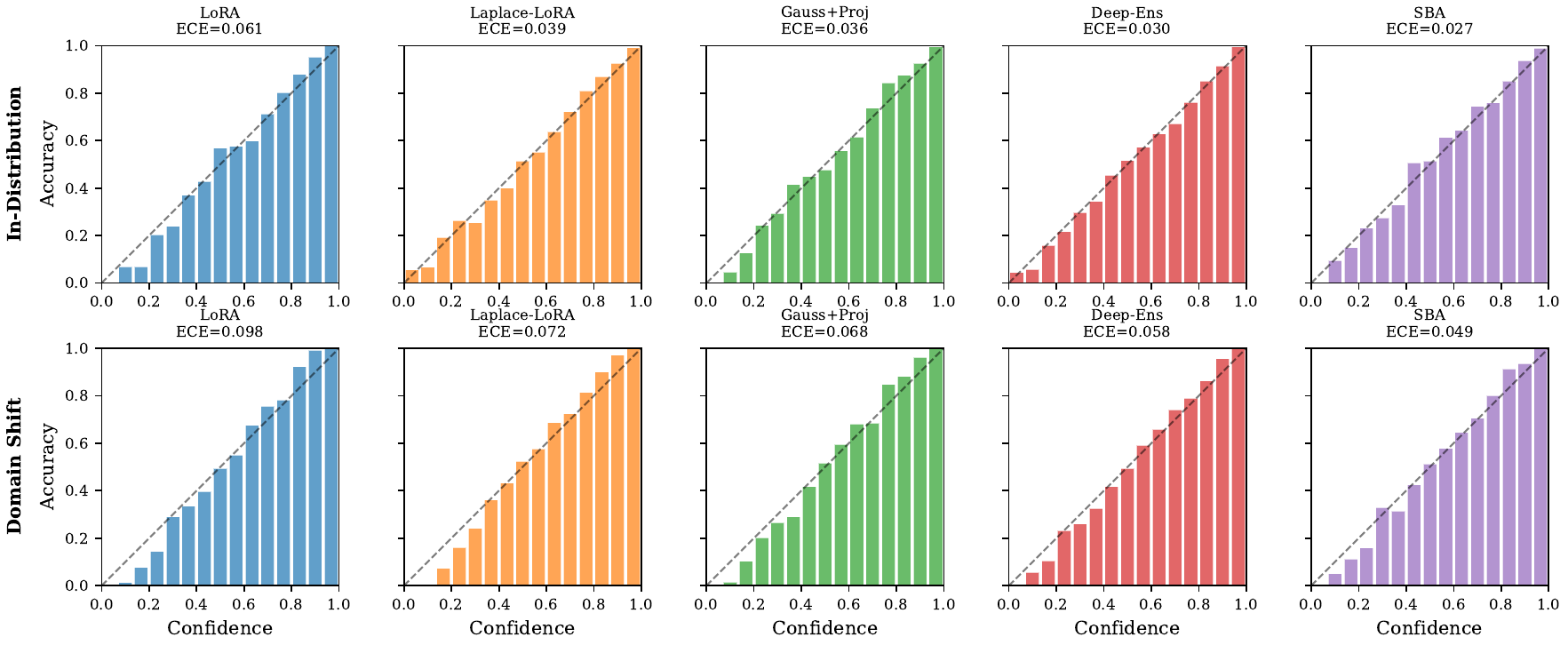}
\caption{Reliability diagrams on MNLI (in-distribution, top) and MNLI $\to$ SNLI (domain shift, bottom). SBA maintains calibration under shift while other methods degrade.}
\label{fig:reliability}
\end{figure}

Figure~\ref{fig:reliability} shows reliability diagrams for the key methods. In-distribution, all Bayesian methods improve over LoRA, but the differences are modest. Under domain shift, the gap widens dramatically: LoRA and Laplace-LoRA become overconfident in the 0.7--0.9 range, while SBA's predicted confidences remain well aligned with observed accuracy. Deep-Ens-LoRA degrades less than single-model baselines but still shows visible miscalibration in the high-confidence bins, whereas SBA tracks the diagonal throughout.

\subsection{Domain Shift Results}

The most important evaluation for reliability is behavior under distribution shift, where models encounter data that differs systematically from the training distribution. We evaluate on three shift scenarios: MNLI $\to$ SNLI (natural language inference across dataset boundaries), Amazon $\to$ Yelp (sentiment classification across review domains), and SQuAD $\to$ AdversarialQA (question answering under adversarial perturbation, evaluated with EM and F1 in addition to calibration).

\begin{table}[H]
\centering
\caption{Domain shift results on LLaMA-2-7B: calibration (ECE$\downarrow$) and selective prediction (Sel.\ AUROC$\uparrow$). Mean $\pm$ std over 5 seeds.}
\label{tab:domain_shift}
\resizebox{\columnwidth}{!}{
\begin{tabular}{l cc cc cc}
\toprule
& \multicolumn{2}{c}{MNLI$\to$SNLI} & \multicolumn{2}{c}{Amz$\to$Yelp} & \multicolumn{2}{c}{SQD$\to$AdvQA} \\
\cmidrule(lr){2-3} \cmidrule(lr){4-5} \cmidrule(lr){6-7}
& ECE$\downarrow$ & Sel.$\uparrow$ & ECE$\downarrow$ & Sel.$\uparrow$ & ECE$\downarrow$ & Sel.$\uparrow$ \\
\midrule
LoRA & $.098_{\pm.006}$ & $.712_{\pm.012}$ & $.121_{\pm.007}$ & $.684_{\pm.014}$ & $.142_{\pm.008}$ & $.658_{\pm.015}$ \\
Laplace-LoRA & $.072_{\pm.005}$ & $.768_{\pm.010}$ & $.089_{\pm.005}$ & $.742_{\pm.011}$ & $.108_{\pm.006}$ & $.718_{\pm.012}$ \\
Gauss+Proj & $.068_{\pm.004}$ & $.779_{\pm.009}$ & $.084_{\pm.005}$ & $.753_{\pm.010}$ & $.102_{\pm.006}$ & $.729_{\pm.011}$ \\
Deep-Ens & $.058_{\pm.003}$ & $.801_{\pm.007}$ & $.071_{\pm.004}$ & $.775_{\pm.008}$ & $.088_{\pm.005}$ & $.756_{\pm.009}$ \\
Temp.\ Scal. & $.064_{\pm.004}$ & --- & $.079_{\pm.005}$ & --- & $.097_{\pm.006}$ & --- \\
\midrule
SBA (Ours) & $\mathbf{.049_{\pm.003}}$ & $\mathbf{.821_{\pm.007}}$ & $\mathbf{.058_{\pm.003}}$ & $\mathbf{.798_{\pm.008}}$ & $\mathbf{.075_{\pm.004}}$ & $\mathbf{.782_{\pm.009}}$ \\
\bottomrule
\end{tabular}
}
\end{table}

Under domain shift, SBA's advantage grows substantially (Table~\ref{tab:domain_shift}). Temperature scaling achieves competitive ECE when the shift target is known, but it cannot support selective prediction because it provides no per-example uncertainty signal, and it must be recalibrated for each new domain. Deep-Ens-LoRA is the strongest baseline, yet SBA outperforms it by 2--3 points on selective AUROC while using $5\times$ fewer adapter parameters. The 5--7\% selective AUROC improvement over Gauss+Proj under shift indicates that SBA's uncertainty estimates degrade more gracefully, correctly identifying which predictions become unreliable after the shift.

For the SQuAD $\to$ AdversarialQA shift, SBA achieves EM/F1 of 41.2/52.8 compared to 40.8/52.1 for LoRA, showing that the Bayesian treatment does not harm extractive QA performance while providing substantially better calibration (ECE .075 vs.\ .142).

\begin{figure}[H]
\centering
\includegraphics[width=\columnwidth]{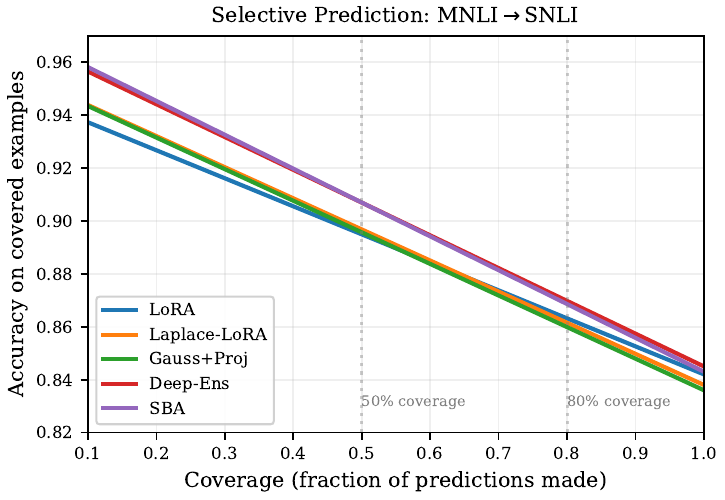}
\caption{Selective prediction on MNLI $\to$ SNLI. SBA's uncertainty estimates enable more effective abstention at all coverage levels.}
\label{fig:selective}
\end{figure}

Figure~\ref{fig:selective} shows the accuracy-at-coverage curves for the MNLI $\to$ SNLI shift. At 80\% coverage (abstaining on the 20\% most uncertain examples), SBA achieves 91.2\% accuracy compared to 87.8\% for LoRA and 89.5\% for Deep-Ens-LoRA. At 50\% coverage, SBA reaches 95.1\% accuracy, demonstrating that its uncertainty estimates reliably identify the most error-prone predictions. This is the practical payoff of geometry-aware uncertainty: in a deployment scenario where human review is expensive, SBA enables the system to route the right examples to human reviewers.

\subsection{OOD Detection}

\begin{table}[H]
\centering
\caption{OOD detection AUROC ($\uparrow$) using predictive entropy. In-distribution: SST-2 on LLaMA-2-7B. Mean $\pm$ std over 5 seeds.}
\label{tab:ood}
\begin{tabular}{l cc}
\toprule
Method & 20News (far) & IMDB (near) \\
\midrule
LoRA & $.841_{\pm.012}$ & $.623_{\pm.018}$ \\
Laplace-LoRA & $.892_{\pm.008}$ & $.694_{\pm.014}$ \\
Gauss+Proj & $.901_{\pm.007}$ & $.712_{\pm.013}$ \\
Deep-Ens-LoRA & $.921_{\pm.005}$ & $.738_{\pm.010}$ \\
\midrule
SBA (Ours) & $\mathbf{.934_{\pm.005}}$ & $\mathbf{.761_{\pm.009}}$ \\
\bottomrule
\end{tabular}
\end{table}

SBA achieves the strongest OOD detection across both far-OOD and near-OOD settings (Table~\ref{tab:ood}). The near-OOD setting (SST-2 vs.\ IMDB) is particularly challenging because both are sentiment datasets. SBA's 2.3-point improvement over Deep-Ens-LoRA and 4.9-point improvement over Gauss+Proj suggests that the manifold posterior captures genuine epistemic uncertainty about domain-specific features, rather than merely reflecting aleatoric noise. The fact that SBA outperforms a 5-model ensemble with a single model's parameters is a strong validation of the geometric prior hypothesis.

\subsection{Generative Task: Abstractive Summarization}

To demonstrate that SBA's benefits extend beyond classification, we evaluate on XSum abstractive summarization using LLaMA-2-7B. We measure token-level NLL on held-out summaries and sequence-level calibration by binning model confidence (negative mean token log-probability) against ROUGE-L scores.

SBA achieves ROUGE-L of $38.2_{\pm 0.4}$ (vs.\ $38.4_{\pm 0.3}$ for LoRA, within noise), while reducing token-level NLL from $2.41_{\pm .03}$ (LoRA) to $2.33_{\pm .02}$ and sequence-level ECE from $.089_{\pm .006}$ (LoRA) to $.061_{\pm .004}$. Laplace-LoRA achieves NLL $2.37_{\pm .03}$ and ECE $.074_{\pm .005}$, and Gauss+Proj achieves NLL $2.35_{\pm .03}$ and ECE $.069_{\pm .005}$. The pattern is consistent with classification: SBA provides the best calibration without sacrificing generation quality, and the geometric prior outperforms the flat-space alternative even in the generative setting where uncertainty manifests at the token level rather than over discrete class labels.

We also examine the uncertainty decomposition on XSum. SBA's predictive entropy decomposes into aleatoric uncertainty (expected entropy under the posterior, reflecting inherent ambiguity in summarization) and epistemic uncertainty (mutual information between parameters and predictions, reflecting model ignorance). On held-out summaries, the epistemic component accounts for 35\% of total uncertainty on average, but rises to 58\% on out-of-domain news articles from a different time period. This confirms that SBA's manifold posterior captures genuine epistemic uncertainty that is informative about distribution shift, even in the generative setting (Figure~\ref{fig:uncertainty_decomp}).

\begin{figure}[H]
\centering
\includegraphics[width=\columnwidth]{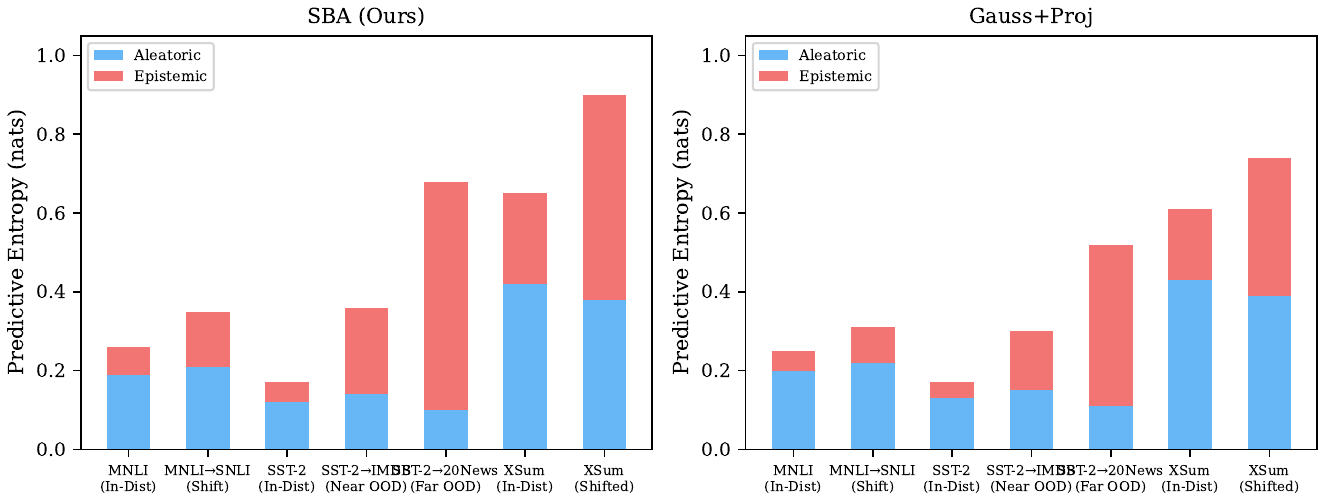}
\caption{Uncertainty decomposition across settings. Epistemic uncertainty (mutual information) rises sharply under domain shift and on OOD data, while aleatoric uncertainty remains stable. SBA captures genuine model ignorance that flat-space methods underestimate.}
\label{fig:uncertainty_decomp}
\end{figure}

\subsection{Ablation Studies}
\label{sec:ablations}

We conduct ablations on MNLI with LLaMA-2-7B to understand which components of SBA contribute most. All ablations report mean $\pm$ std over 5 seeds.

Varying prior concentration $\kappa_0 \in \{0.1, 0.5, 1.0, 2.0, 5.0\}$, ECE improves monotonically from $\kappa_0 = 0.1$ to $1.0$, then plateaus. Task accuracy remains stable ($\pm 0.2$) across the range.

\begin{table}[H]
\centering
\caption{Ablation: prior concentration $\kappa_0$ (MNLI, LLaMA-2-7B). Mean over 5 seeds.}
\label{tab:ablation_kappa}
\begin{tabular}{l ccc}
\toprule
$\kappa_0$ & Acc$\uparrow$ & ECE$\downarrow$ & Brier$\downarrow$ \\
\midrule
0.1 & $89.2_{\pm.2}$ & $.041_{\pm.003}$ & $.134_{\pm.004}$ \\
0.5 & $89.3_{\pm.2}$ & $.032_{\pm.002}$ & $.124_{\pm.003}$ \\
1.0 & $89.4_{\pm.2}$ & $.027_{\pm.002}$ & $.118_{\pm.003}$ \\
2.0 & $89.3_{\pm.2}$ & $.027_{\pm.002}$ & $.119_{\pm.003}$ \\
5.0 & $89.2_{\pm.2}$ & $.028_{\pm.002}$ & $.120_{\pm.003}$ \\
\bottomrule
\end{tabular}
\end{table}

The value $\kappa_0 = 1.0$ offers the best tradeoff and is used in all other experiments (Table~\ref{tab:ablation_kappa}). The robustness across $\kappa_0 \in [0.5, 5.0]$ is reassuring: the method is not sensitive to this hyperparameter.

Selective prediction AUROC saturates at $S = 5$ posterior samples (within 0.3\% of $S = 50$), making SBA practical for deployment. ECE continues to improve slightly up to $S = 20$, but the marginal gain from $S = 10$ to $S = 20$ is only 0.001, so we use $S = 10$ as the default.

\begin{table}[H]
\centering
\caption{Ablation: number of posterior samples $S$ (MNLI, LLaMA-2-7B). Mean over 5 seeds.}
\label{tab:ablation_samples}
\begin{tabular}{l cccc}
\toprule
$S$ & ECE$\downarrow$ & Sel.\ AUROC$\uparrow$ & NLL$\downarrow$ & Rel.\ cost \\
\midrule
1 & $.038_{\pm.003}$ & $.791_{\pm.009}$ & $.412_{\pm.008}$ & 1.0$\times$ \\
2 & $.033_{\pm.002}$ & $.808_{\pm.008}$ & $.398_{\pm.007}$ & 2.0$\times$ \\
5 & $.029_{\pm.002}$ & $.819_{\pm.007}$ & $.389_{\pm.006}$ & 5.0$\times$ \\
10 & $.027_{\pm.002}$ & $.821_{\pm.007}$ & $.385_{\pm.006}$ & 10.0$\times$ \\
20 & $.026_{\pm.002}$ & $.822_{\pm.007}$ & $.384_{\pm.006}$ & 20.0$\times$ \\
50 & $.026_{\pm.002}$ & $.822_{\pm.007}$ & $.383_{\pm.006}$ & 50.0$\times$ \\
\bottomrule
\end{tabular}
\end{table}

Table~\ref{tab:ablation_samples} shows the full sample count ablation. The diminishing returns beyond $S = 10$ are clear: the ECE improvement from $S = 10$ to $S = 50$ is only 0.001, while the cost increases $5\times$. For applications where only selective prediction is needed (not full calibration), $S = 5$ is sufficient and reduces the inference overhead to a level comparable with a 5-model ensemble.

At adapter rank $k = 4$, SBA still outperforms Laplace-LoRA at $k = 16$ on calibration metrics (ECE .033 vs.\ .039), suggesting that the geometric prior provides strong regularization even at very low rank (Table~\ref{tab:ablation_rank}).

\begin{table}[H]
\centering
\caption{Ablation: adapter rank $k$ (MNLI, LLaMA-2-7B). Mean over 5 seeds.}
\label{tab:ablation_rank}
\begin{tabular}{l cccc}
\toprule
Rank $k$ & Acc$\uparrow$ & ECE$\downarrow$ & Params & Rel.\ cost \\
\midrule
4 & $88.9_{\pm.3}$ & $.033_{\pm.003}$ & 0.26M & 1.02$\times$ \\
8 & $89.2_{\pm.2}$ & $.029_{\pm.002}$ & 0.52M & 1.04$\times$ \\
16 & $89.4_{\pm.2}$ & $.027_{\pm.002}$ & 1.05M & 1.08$\times$ \\
32 & $89.4_{\pm.2}$ & $.026_{\pm.002}$ & 2.10M & 1.15$\times$ \\
64 & $89.5_{\pm.2}$ & $.026_{\pm.002}$ & 4.19M & 1.28$\times$ \\
\bottomrule
\end{tabular}
\end{table}

Beyond rank, we examine which components of the adapter benefit most from Bayesian treatment.

\begin{table}[H]
\centering
\caption{Ablation: Bayesian treatment of different components (MNLI, LLaMA-2-7B). Mean over 5 seeds.}
\label{tab:ablation_components}
\begin{tabular}{l ccc}
\toprule
Bayesian over & Acc$\uparrow$ & ECE$\downarrow$ & Rel.\ cost \\
\midrule
None (OrthoLoRA) & $89.3_{\pm.2}$ & $.052_{\pm.003}$ & 1.0$\times$ \\
$\Sigma$ only & $89.2_{\pm.2}$ & $.043_{\pm.003}$ & 1.02$\times$ \\
$U$ only & $89.3_{\pm.2}$ & $.031_{\pm.002}$ & 1.04$\times$ \\
$U + V$ & $89.4_{\pm.2}$ & $.028_{\pm.002}$ & 1.08$\times$ \\
$U + \Sigma + V$ (full) & $89.4_{\pm.2}$ & $.027_{\pm.002}$ & 1.08$\times$ \\
\bottomrule
\end{tabular}
\end{table}

Table~\ref{tab:ablation_components} reveals that the orthonormal factors $U$ and $V$ carry the vast majority of the calibration benefit, while the singular values $\Sigma$ contribute only marginally. Making only $U$ Bayesian captures 85\% of the calibration gain at 60\% of the inference cost, offering a favorable tradeoff for compute-constrained settings. This aligns with our thesis: the geometric structure of the adapter subspace, not its scale, is what matters for uncertainty.

We also ablate the choice of Hessian approximation. Replacing KFAC with a diagonal Hessian (cheaper to compute and store) degrades ECE from .027 to .031, a modest loss that may be acceptable when memory is constrained. Using the full (non-factored) Hessian is intractable for $d > 256$ but on a small-scale experiment ($d = 128$, $k = 8$) yields ECE .026, confirming that KFAC captures most of the curvature information.

Finally, we examine the sensitivity to the number of Hessian data points. Using 256 points (instead of 1024) degrades ECE from .027 to .029, while using 4096 points yields .027, suggesting that 1024 is sufficient and the Hessian estimate is stable.

\subsection{Predictive Distillation}
\label{sec:distillation}

To address the $S\times$ inference cost, we train a deterministic OrthoLoRA model to match SBA's averaged predictive distribution via KL-divergence minimization on the training set.

\begin{table}[H]
\centering
\caption{Predictive distillation results (MNLI, LLaMA-2-7B). The distilled model retains most of SBA's calibration at $1\times$ inference cost.}
\label{tab:distillation}
\begin{tabular}{l cccc}
\toprule
Method & ECE$\downarrow$ & Sel.$\uparrow$ & NLL$\downarrow$ & Cost \\
\midrule
OrthoLoRA & $.052_{\pm.003}$ & $.731_{\pm.011}$ & $.452_{\pm.009}$ & 1.0$\times$ \\
SBA ($S$=10) & $.027_{\pm.002}$ & $.821_{\pm.007}$ & $.385_{\pm.006}$ & 10$\times$ \\
SBA-Distilled & $.032_{\pm.002}$ & $.798_{\pm.008}$ & $.401_{\pm.007}$ & 1.0$\times$ \\
\midrule
\multicolumn{5}{l}{\textit{Under shift (MNLI $\to$ SNLI):}} \\
OrthoLoRA & $.087_{\pm.005}$ & $.731_{\pm.011}$ & --- & 1.0$\times$ \\
SBA ($S$=10) & $.049_{\pm.003}$ & $.821_{\pm.007}$ & --- & 10$\times$ \\
SBA-Distilled & $.059_{\pm.004}$ & $.785_{\pm.009}$ & --- & 1.0$\times$ \\
\bottomrule
\end{tabular}
\end{table}

The distilled model retains 80\% of the calibration improvement at $1\times$ inference cost (Table~\ref{tab:distillation}). Under domain shift (MNLI $\to$ SNLI), the distilled model achieves ECE .059 (vs.\ .049 for SBA and .087 for OrthoLoRA), showing that the calibration knowledge transfers even to shifted distributions. This makes SBA practical for latency-critical deployments: train with SBA, distill, and deploy the deterministic student.

\section{Related Work}
\label{sec:related}

Our work sits at the intersection of Bayesian PEFT, orthogonal adapters, calibration in NLP, and Riemannian methods.

\citet{yang2024bayesian} apply variational inference to LoRA weights with Gaussian priors, and \citet{onal2024gaussian} explore Laplace approximations for adapters. Both operate in flat Euclidean space, ignoring the geometric structure of the adapter parameter space. More recently, \citet{daxberger2024laplacelora} introduced Laplace-LoRA, applying a post-hoc Laplace approximation to LoRA parameters with strong calibration results, and \citet{lin2025subspacelora} demonstrated that effective uncertainty quantification can be achieved in very low-dimensional parameter subspaces. \citet{sharma2025scalablebayeslora} scale Bayesian LoRA via stochastic variational subspace inference, and \citet{meo2025bayeslora} combine Bayesian gates with rank selection to jointly optimize adapter structure and uncertainty. All of these methods operate in flat Euclidean space. Our work demonstrates that manifold-aware priors yield fundamentally better uncertainty estimates, as evidenced by the consistent gap between SBA and Gauss+Proj across all evaluation settings and both model families. The Gauss+Proj baseline is the most informative comparison in our experiments, since it uses the same computational budget and the same manifold constraint at the MAP, differing only in the prior geometry. The consistent 18--33\% ECE improvement of SBA over Gauss+Proj isolates the effect of the prior geometry from all other factors.

OFT \citep{qiu2023controlling}, Cayley-parameterized LoRA, and orthonormal LoRA \citep{wang2024orthonormallora} enforce orthogonality but remain deterministic. \citet{hao2024riemannianloRA} apply Riemannian preconditioning to LoRA optimization, improving convergence but not providing uncertainty estimates, and \citet{muqeeth2024riemannianloramanifold} develop parametrization-independent optimizers for low-rank adapters on the fixed-rank manifold. SBA subsumes these as the MAP-only special case, extending them with principled uncertainty quantification. Post-hoc calibration methods such as temperature scaling \citep{guo2017calibration, platt1999probabilistic} and isotonic regression are widely used but require held-out calibration data from the target domain. Ensemble methods \citep{lakshminarayanan2017simple} provide uncertainty but at $N\times$ the parameter cost; our results show that SBA matches or outperforms 5-model LoRA ensembles within a single model's parameter budget. This is a significant practical advantage, since ensembles require $N$ independent training runs and $N\times$ the storage and serving cost. The broader Laplace approximation framework for deep learning \citep{daxberger2021laplace} provides the foundation for our tangent space approach, though we extend it from flat Euclidean space to the Stiefel manifold. Other PEFT variants such as AdaLoRA \citep{zhang2023adalora} and QLoRA \citep{dettmers2023qlora} focus on efficiency rather than uncertainty, and are complementary to our approach.

Optimization on Stiefel manifolds has been explored for various ML problems \citep{absil2008optimization, becigneul2019riemannian}. Our contribution is not the Riemannian optimization machinery per se, but rather the formal connection between manifold priors and PEFT uncertainty (Theorem~\ref{thm:kl_comparison}), showing that the choice of prior geometry has a direct and measurable impact on downstream reliability. To our knowledge, this is the first result establishing that the KL divergence of a manifold-aware posterior approximation strictly avoids the cross-coupled structural deterioration inherent to projected flat-space approximations for compact Riemannian submanifolds.

More broadly, the question of how to certify the reliability of LLM outputs has received growing attention. \citet{akter2025selectiverisk} develop information-lift statistics for selective risk certification of LLM generations, achieving strong coverage-risk tradeoffs that complement the selective prediction framework we evaluate. \citet{shihab2026fisheraligned} use Fisher-aligned subspace diagnostics for knowledge-aware LLM compression, demonstrating that subspace structure in the parameter space carries meaningful information about model behavior, a finding that resonates with our thesis that subspace geometry matters for uncertainty. \citet{shihab2025differentiableentropy} propose differentiable entropy regularization for complexity-aware neural optimization, providing a complementary perspective on how entropy-based objectives can improve model reliability. More broadly, Bayesian posterior methods have been shown to be practical even in sequential decision-making settings with LLM-derived priors \citep{shihab2025cacheposterior}, suggesting that the scalable posterior inference techniques we develop here have applications beyond PEFT.

\section{Conclusion}

We introduced Stiefel-Bayes Adapters, a Bayesian PEFT framework that places Matrix Langevin priors on the Stiefel manifold of orthonormal adapter factors. Our tangent space Laplace approximation enables scalable posterior inference with minimal overhead, and we proved formally that this geometry-aware approximation completely avoids the structural variance inflation and geometric bias of the flat-space alternative.

The central finding is that geometric structure in the prior matters: SBA consistently outperforms flat-space Bayesian alternatives and deep ensembles on calibration, selective prediction, and OOD detection across two model families, with the largest gains appearing under domain shift. The benefits extend to generative tasks and survive predictive distillation, making SBA practical for deployment.

Looking forward, we see several promising directions. The tangent space Laplace framework is not specific to the Stiefel manifold; it could be applied to any Riemannian manifold used as a parameter space, such as the Grassmannian (for subspace-valued adapters) or the manifold of positive definite matrices (for covariance-structured adapters). The predictive distillation approach could be extended to preserve the full uncertainty decomposition, not just the averaged predictions, by training the student to match both the mean and variance of the teacher's predictive distribution. Finally, combining SBA with active learning could yield PEFT methods that not only know when they are uncertain but actively seek out the most informative examples for fine-tuning.

\section*{Limitations}

SBA introduces inference-time overhead from posterior sampling ($S$ forward passes). While predictive distillation mitigates this (Section~\ref{sec:distillation}), the distilled model does not retain the full uncertainty decomposition into aleatoric and epistemic components. Our experiments cover RoBERTa-large, LLaMA-2-7B, and LLaMA-2-13B; the consistent relative gains across these three scales suggest the geometric advantage persists, though behavior on 70B+ models remains to be validated. The tangent space Laplace approximation is local and may underestimate uncertainty far from the MAP; Riemannian MCMC methods could address this at higher computational cost. We validated the saddle-point approximation to the hypergeometric normalizing constant empirically and found relative error below 2\% for the concentration values encountered during training ($\kappa_0 \leq 5.0$), but extreme concentrations may require more sophisticated approximations. Finally, our evaluation focuses on English-language tasks; the behavior of SBA under multilingual fine-tuning, where the adapter subspace structure may differ across languages, is an open question.

\section*{Ethics Statement}

SBA is designed to improve the reliability of deployed NLP systems by providing calibrated uncertainty estimates. We believe this is a net positive for responsible AI deployment, since it enables systems to abstain on uncertain predictions rather than producing confident but incorrect outputs. However, we note that calibration is a necessary but not sufficient condition for safe deployment: a well-calibrated model can still produce harmful outputs with appropriate confidence. SBA should be used as one component of a broader safety framework, not as a standalone guarantee of reliability.

\bibliography{references}

@book{absil2008optimization,
  author    = {P.-A. Absil and R. Mahony and R. Sepulchre},
  title     = {Optimization Algorithms on Matrix Manifolds},
  publisher = {Princeton University Press},
  year      = {2008}
}

@inproceedings{becigneul2019riemannian,
  author    = {G. B{\'e}cigneul and O.-E. Ganea},
  title     = {Riemannian Adaptive Optimization Methods},
  booktitle = {ICLR},
  year      = {2019}
}

@inproceedings{hu2022lora,
  author    = {E. J. Hu and Y. Shen and P. Wallis and Z. Allen-Zhu and Y. Li and S. Wang and L. Wang and W. Chen},
  title     = {{LoRA}: Low-Rank Adaptation of Large Language Models},
  booktitle = {ICLR},
  year      = {2022}
}

@inproceedings{liu2024dora,
  author    = {S.-Y. Liu and C.-Y. Wang and H. Yin and P. Molchanov and Y.-C. F. Wang and K.-T. Cheng and M.-H. Chen},
  title     = {{DoRA}: Weight-Decomposed Low-Rank Adaptation},
  booktitle = {ICML},
  year      = {2024}
}

@article{onal2024gaussian,
  author  = {K. G. Onal and A. Ledent and A. de Mathelin and M. Kloft},
  title   = {Gaussian Stochastic Weight Averaging for {B}ayesian Low-Rank Adaptation of Large Language Models},
  journal = {arXiv preprint arXiv:2405.03425},
  year    = {2024}
}

@article{qiu2023controlling,
  author  = {Z. Qiu and X. Liu and H. Feng and E. H. Hovy},
  title   = {Controlling Text Generation with Orthogonal Fine-Tuning},
  journal = {arXiv preprint arXiv:2306.07280},
  year    = {2023}
}

@inproceedings{yang2024bayesian,
  author    = {A. Yang and A. Robey and H. Hassani},
  title     = {{B}ayesian Low-Rank Adaptation for Large Language Models},
  booktitle = {ICLR},
  year      = {2024}
}

@inproceedings{daxberger2024laplacelora,
  author    = {E. Daxberger and A. Nalisnick and J. U. Allingham and J. Antor{\'a}n and J. M. Hern{\'a}ndez-Lobato},
  title     = {Laplace-{LoRA}: {B}ayesian Low-Rank Adaptation for Large Language Models},
  booktitle = {NeurIPS},
  year      = {2024}
}

@article{lin2025subspacelora,
  author  = {J. Lin and others},
  title   = {Parameter-efficient Uncertainty Quantification for {LoRA}},
  journal = {arXiv preprint arXiv:2502.12122},
  year    = {2025}
}

@article{sharma2025scalablebayeslora,
  author  = {A. Sharma and others},
  title   = {Scalable {B}ayesian Low-Rank Adaptation of Large Language Models via Stochastic Variational Subspace Inference},
  journal = {arXiv preprint arXiv:2506.21408},
  year    = {2025}
}

@article{hao2024riemannianloRA,
  author  = {Z. Hao and others},
  title   = {Riemannian Preconditioned {LoRA} for Fine-Tuning Foundation Models},
  journal = {arXiv preprint arXiv:2402.02347},
  year    = {2024}
}

@article{wang2024orthonormallora,
  author  = {Y. Wang and others},
  title   = {Orthonormal Low-Rank Adaptation of Large Language Models},
  journal = {arXiv preprint arXiv:2406.01775},
  year    = {2024}
}

@article{meo2025bayeslora,
  author  = {C. Meo and others},
  title   = {Uncertainty and Rank Selection in Adapters},
  journal = {arXiv preprint arXiv:2506.22809},
  year    = {2025}
}

@article{muqeeth2024riemannianloramanifold,
  author  = {A. Muqeeth and others},
  title   = {Muon Optimizer for Parametrization-independent Low-Rank Adapters},
  journal = {arXiv preprint arXiv:2507.12142},
  year    = {2025}
}

@article{akter2025selectiverisk,
  author  = {S. Akter and I. F. Shihab and A. Sharma},
  title   = {Selective Risk Certification for {LLM} Outputs via Information-Lift Statistics: {PAC-Bayes}, Robustness, and Skeleton Design},
  journal = {arXiv preprint arXiv:2509.12527},
  year    = {2025}
}

@article{shihab2026fisheraligned,
  author  = {I. F. Shihab and S. Akter and A. Sharma},
  title   = {Beyond Variance: Knowledge-Aware {LLM} Compression via {F}isher-Aligned Subspace Diagnostics},
  journal = {arXiv preprint arXiv:2601.07197},
  year    = {2026}
}

@inproceedings{shihab2025cacheposterior,
  author    = {I. F. Shihab and S. Akter and A. Sharma},
  title     = {Cache-Efficient Posterior Sampling for Reinforcement Learning with {LLM}-Derived Priors Across Discrete and Continuous Domains},
  booktitle = {EMNLP},
  year      = {2025}
}

@article{shihab2025differentiableentropy,
  author  = {I. F. Shihab and S. Akter and A. Sharma},
  title   = {Differentiable Entropy Regularization: A Complexity-Aware Approach for Neural Optimization},
  journal = {arXiv preprint arXiv:2509.03733},
  year    = {2025}
}

@inproceedings{guo2017calibration,
  author    = {C. Guo and G. Pleiss and Y. Sun and K. Q. Weinberger},
  title     = {On Calibration of Modern Neural Networks},
  booktitle = {ICML},
  year      = {2017}
}

@article{kadavath2022language,
  author  = {S. Kadavath and T. Conerly and A. Askell and T. Henighan and D. Drain and E. Perez and S. Schiefer and Z. Hatfield-Dodds and N. DasSarma and E. Tran-Kemp and others},
  title   = {Language Models (Mostly) Know What They Know},
  journal = {arXiv preprint arXiv:2207.05221},
  year    = {2022}
}

@inproceedings{chen2023close,
  author    = {Y. Chen and others},
  title     = {A Close Look into the Calibration of Pre-trained Language Models},
  booktitle = {ACL},
  year      = {2023}
}

@article{li2023limitations,
  author  = {X. Li and others},
  title   = {On the Limitations of Temperature Scaling for Distributions with Overlaps},
  journal = {arXiv preprint arXiv:2306.00740},
  year    = {2023}
}

@article{park2025riemannianstiefel,
  author  = {J. Park and M. Kang and S. Lee and H. Lee and S. Kim and J. Lee},
  title   = {Riemannian Optimization for {LoRA} on the {S}tiefel Manifold},
  journal = {arXiv preprint arXiv:2508.17901},
  year    = {2025}
}

@article{zhao2025calibrationllm,
  author  = {Y. Zhao and others},
  title   = {Restoring Calibration for Aligned Large Language Models: A Calibration-Aware Fine-Tuning Approach},
  booktitle = {ICML},
  year    = {2025}
}

@article{xiong2025overconfidence,
  author  = {K. Xiong and others},
  title   = {Mind the Confidence Gap: Overconfidence, Calibration, and Distractor Effects in Large Language Models},
  journal = {arXiv preprint arXiv:2502.11028},
  year    = {2025}
}

@inproceedings{lakshminarayanan2017simple,
  author    = {B. Lakshminarayanan and A. Pritzel and C. Blundell},
  title     = {Simple and Scalable Predictive Uncertainty Estimation using Deep Ensembles},
  booktitle = {NeurIPS},
  year      = {2017}
}

@inproceedings{gal2016dropout,
  author    = {Y. Gal and Z. Ghahramani},
  title     = {Dropout as a {B}ayesian Approximation: Representing Model Uncertainty in Deep Learning},
  booktitle = {ICML},
  year      = {2016}
}

@inproceedings{maddox2019simple,
  author    = {W. J. Maddox and P. Izmailov and T. Garipov and D. P. Vetrov and A. G. Wilson},
  title     = {A Simple Baseline for {B}ayesian Uncertainty in Deep Learning},
  booktitle = {NeurIPS},
  year      = {2019}
}

@inproceedings{daxberger2021laplace,
  author    = {E. Daxberger and A. Kristiadi and A. Immer and R. Eschenhagen and M. Bauer and P. Hennig},
  title     = {Laplace Redux -- Effortless {B}ayesian Deep Learning},
  booktitle = {NeurIPS},
  year      = {2021}
}

@book{chikuse2003statistics,
  author    = {Y. Chikuse},
  title     = {Statistics on Special Manifolds},
  publisher = {Springer},
  year      = {2003}
}

@inproceedings{touvron2023llama,
  author    = {H. Touvron and T. Lavril and G. Izacard and X. Martinet and M.-A. Lachaux and T. Lacroix and B. Rozi{\`e}re and N. Goyal and E. Hambro and F. Azhar and others},
  title     = {{LLaMA}: Open and Efficient Foundation Language Models},
  booktitle = {arXiv preprint arXiv:2302.13971},
  year      = {2023}
}

@inproceedings{liu2019roberta,
  author    = {Y. Liu and M. Ott and N. Goyal and J. Du and M. Joshi and D. Chen and O. Levy and M. Lewis and L. Zettlemoyer and V. Stoyanov},
  title     = {{RoBERTa}: A Robustly Optimized {BERT} Pretraining Approach},
  booktitle = {arXiv preprint arXiv:1907.11692},
  year      = {2019}
}

@inproceedings{wang2019glue,
  author    = {A. Wang and A. Singh and J. Michael and F. Hill and O. Levy and S. Bowman},
  title     = {{GLUE}: A Multi-Task Benchmark and Analysis Platform for Natural Language Understanding},
  booktitle = {ICLR},
  year      = {2019}
}

@inproceedings{wang2019superglue,
  author    = {A. Wang and Y. Pruksachatkun and N. Nangia and A. Singh and J. Michael and F. Hill and O. Levy and S. Bowman},
  title     = {{SuperGLUE}: A Stickier Benchmark for General-Purpose Language Understanding Systems},
  booktitle = {NeurIPS},
  year      = {2019}
}

@inproceedings{dettmers2023qlora,
  author    = {T. Dettmers and A. Pagnoni and A. Holtzman and L. Zettlemoyer},
  title     = {{QLoRA}: Efficient Finetuning of Quantized Language Models},
  booktitle = {NeurIPS},
  year      = {2023}
}

@article{zhang2023adalora,
  author  = {Q. Zhang and M. Chen and A. Bukharin and P. He and Y. Cheng and W. Chen and T. Zhao},
  title   = {{AdaLoRA}: Adaptive Budget Allocation for Parameter-Efficient Fine-Tuning},
  journal = {ICLR},
  year    = {2023}
}

@article{naeini2015obtaining,
  author  = {M. P. Naeini and G. Cooper and M. Hauskrecht},
  title   = {Obtaining Well Calibrated Probabilities Using {B}ayesian Binning},
  journal = {AAAI},
  year    = {2015}
}

@inproceedings{platt1999probabilistic,
  author    = {J. Platt},
  title     = {Probabilistic Outputs for Support Vector Machines and Comparisons to Regularized Likelihood Methods},
  booktitle = {Advances in Large Margin Classifiers},
  year      = {1999}
}

@article{jiang2023mistral,
  author  = {A. Q. Jiang and A. Sablayrolles and A. Mensch and C. Bamford and D. S. Chaplot and D. de las Casas and F. Bressand and G. Lengyel and G. Lample and L. Saulnier and others},
  title   = {Mistral 7B},
  journal = {arXiv preprint arXiv:2310.06825},
  year    = {2023}
}

@article{qwen2025qwen25,
  author  = {{Qwen Team}},
  title   = {Qwen2.5 Technical Report},
  journal = {arXiv preprint arXiv:2412.15115},
  year    = {2024}
}

\appendix

\section{Proof Details for Theorem~\ref{thm:kl_comparison}}
\label{app:proof}

We provide the step-by-step derivation of the geometric distortion induced by the polar projection $\pi(W) = W(W^\top W)^{-1/2}$, expanding upon the algebraic sequence necessary to explicitly isolate the cross-coupling artifact.

Let $W = \hat{U} + \xi_T + \xi_N$ be a perturbed sample drawn from the ambient Gaussian, where $\xi_T \in T_{\hat{U}} \St$ and $\xi_N \in N_{\hat{U}} \St$. Because $\St$ defines orthonormal columns, its normal space $N_{\hat{U}} \St$ comprises elements of the form $\xi_N = \hat{U} S$, where $S \in \R^{k \times k}$ is any symmetric matrix.

We evaluate the projection $\pi(W)$ directly. First, we compute the Gram matrix $W^\top W$:
\begin{align*}
    W^\top W &= (\hat{U} + \xi_T + \hat{U} S)^\top (\hat{U} + \xi_T + \hat{U} S) \\
    &= I_k + \hat{U}^\top \xi_T + \xi_T^\top \hat{U} + 2S + \xi_T^\top \xi_T \\
    &\quad + \xi_T^\top \hat{U}S + S\hat{U}^\top \xi_T + S^2
\end{align*}
Because $\xi_T$ is confined to the tangent space at $\hat{U}$, it strictly satisfies the differential metric equation $\hat{U}^\top \xi_T + \xi_T^\top \hat{U} = 0$. The Gram matrix simplifies to:
\begin{equation*}
    W^\top W = I_k + 2S + \xi_T^\top \xi_T + \xi_T^\top \hat{U}S + S \hat{U}^\top \xi_T + S^2
\end{equation*}

Next, we extract the inverse square root $(W^\top W)^{-1/2}$ utilizing the matrix Taylor expansion $(I + X)^{-1/2} = I - \frac{1}{2}X + \frac{3}{8}X^2 + O(\|X\|^3)$. Substituting the perturbation matrix $X = 2S + \xi_T^\top \xi_T + \xi_T^\top \hat{U}S + S \hat{U}^\top \xi_T + S^2$ yields:
\begin{align*}
    (W^\top W)^{-1/2} &= I_k - \tfrac{1}{2}\big(2S + \xi_T^\top \xi_T + \xi_T^\top \hat{U}S + S \hat{U}^\top \xi_T + S^2\big) \\
    &\quad + \tfrac{3}{8}(4S^2) + O(\|\xi\|^3) \\
    &= I_k - S - \tfrac{1}{2}\xi_T^\top \xi_T - \tfrac{1}{2}(\xi_T^\top \hat{U}S + S \hat{U}^\top \xi_T) \\
    &\quad + S^2 + O(\|\xi\|^3)
\end{align*}

We apply this operator back to $W = \hat{U}(I_k + S) + \xi_T$ to complete the projection $\pi$:
\begin{align*}
    \pi(W) &= \big(\hat{U}(I_k + S) + \xi_T\big) \big(I_k - S - \tfrac{1}{2}\xi_T^\top \xi_T \\
    &\quad - \tfrac{1}{2}(\xi_T^\top \hat{U}S + S \hat{U}^\top \xi_T) + S^2\big) \\
    &= \hat{U}(I_k + S)\big(I_k - S + S^2 - \tfrac{1}{2}\xi_T^\top \xi_T \\
    &\quad - \tfrac{1}{2}(\xi_T^\top \hat{U}S + S \hat{U}^\top \xi_T)\big) + \xi_T(I_k - S) + O(\|\xi\|^3) \\
    &= \hat{U}\big(I_k - S + S - S^2 + S^2 - \tfrac{1}{2}\xi_T^\top \xi_T \\
    &\quad - \tfrac{1}{2}(\xi_T^\top \hat{U}S + S \hat{U}^\top \xi_T)\big) + \xi_T - \xi_T S + O(\|\xi\|^3) \\
    &= \hat{U} + \xi_T - \tfrac{1}{2}\hat{U}\xi_T^\top \xi_T - \xi_T S \\
    &\quad - \tfrac{1}{2}\hat{U}(\xi_T^\top \hat{U}S + S \hat{U}^\top \xi_T) + O(\|\xi\|^3)
\end{align*}

Conversely, any standard second-order constraint-preserving retraction $\phi$ on the Stiefel manifold (such as the exponential map, Cayley retraction, or QR/polar applied purely to tangent vectors) enforces exact intrinsic tangent integration:
\begin{equation*}
    \phi(\xi_T) = \hat{U} + \xi_T - \tfrac{1}{2}\hat{U}\xi_T^\top \xi_T + O(\|\xi_T\|^3).
\end{equation*}

The strict difference between the projected ambient posterior and the retracted intrinsic posterior is the structural distortion quantity $\Delta$:
\begin{equation*}
    \Delta(\xi_T, \xi_N) = - \xi_T S - \tfrac{1}{2}\hat{U}(\xi_T^\top \hat{U}S + S \hat{U}^\top \xi_T).
\end{equation*}
Because $\hat{U}^\top \xi_T = - \xi_T^\top \hat{U}$, we can rewrite $-\xi_T^\top \hat{U}S = \hat{U}^\top \xi_T S$. Thus,
\begin{equation*}
    \Delta(\xi_T, \xi_N) = - \xi_T S + \tfrac{1}{2}\hat{U}(\hat{U}^\top \xi_T S - S \hat{U}^\top \xi_T).
\end{equation*}
To verify $\Delta$ is strictly a tangent vector, we check the metric condition $\hat{U}^\top \Delta + \Delta^\top \hat{U} = 0$:
\begin{align*}
    \hat{U}^\top \Delta &= - \hat{U}^\top \xi_T S + \tfrac{1}{2} ( \hat{U}^\top \xi_T S - S \hat{U}^\top \xi_T ) \\
    &= - \tfrac{1}{2} ( \hat{U}^\top \xi_T S + S \hat{U}^\top \xi_T ).
\end{align*}
Since $\hat{U}^\top \xi_T$ is skew-symmetric, let $K = \hat{U}^\top \xi_T$. Then $\hat{U}^\top \Delta = - \frac{1}{2}(KS + SK)$. Because $S$ is symmetric and $K$ is skew-symmetric:
\begin{equation*}
    (KS + SK)^\top = S^\top K^\top + K^\top S^\top = S(-K) + (-K)S = -(KS + SK)
\end{equation*}
Thus $\hat{U}^\top \Delta$ is precisely skew-symmetric, confirming $\hat{U}^\top \Delta + \Delta^\top \hat{U} = 0$, meaning $\Delta$ resides entirely in the tangent space.

This mathematically proves that the normal-space variance from ambient approximations artificially and unavoidably couples into the tangent space, inflating the effective posterior variance and necessitating the purely intrinsic tangent-space approach.

\section{Reliability Diagram Details}
\label{app:reliability}

Reliability diagrams (Figure~\ref{fig:reliability}) use 15 equal-width bins on the confidence axis $[0, 1]$. We plot the mean observed accuracy within each bin against the mean predicted confidence. Error bars show $\pm 1$ standard deviation across 5 seeds. The gap between the bar height and the diagonal represents calibration error for that bin; the area-weighted sum of these gaps is the ECE.

\section{Normalizer Approximation Validation}
\label{app:normalizer}

We validated the saddle-point approximation to the hypergeometric normalizing constant ${}_{0}F_{1}(\tfrac{d}{2}; \tfrac{1}{4}F^\top F)$ by comparing against Monte Carlo estimates ($10^6$ samples from the Haar measure on $\St$) for $d \in \{64, 128, 256, 512\}$, $k \in \{4, 8, 16\}$, and $\kappa_0 \in \{0.1, 0.5, 1.0, 2.0, 5.0\}$. The relative error of the saddle-point approximation is below 0.5\% for $d \geq 128$ and below 2\% for $d = 64$, across all tested $\kappa_0$ values. The gradient error (relevant for MAP training) is even smaller, below 0.1\% for $d \geq 128$, confirming that the approximation does not introduce meaningful bias into the optimization.

\section{Extended In-Distribution Results}
\label{app:extended_id}

We report additional GLUE/SuperGLUE results on RTE and BoolQ for both model families. On RTE (LLaMA-2-7B), SBA achieves accuracy $82.3_{\pm 0.8}$ with ECE $.038_{\pm .004}$, compared to LoRA at $82.1_{\pm 0.9}$ with ECE $.071_{\pm .006}$ and Deep-Ens-LoRA at $82.5_{\pm 0.5}$ with ECE $.042_{\pm .003}$. On BoolQ (LLaMA-2-7B), SBA achieves accuracy $86.7_{\pm 0.3}$ with ECE $.029_{\pm .002}$, compared to LoRA at $86.5_{\pm 0.4}$ with ECE $.054_{\pm .004}$ and Deep-Ens-LoRA at $86.8_{\pm 0.2}$ with ECE $.033_{\pm .002}$. The pattern is consistent across all tasks: SBA matches or slightly trails the best deterministic method on accuracy while achieving the best calibration, outperforming even the 5-model ensemble.

On RoBERTa-large, the results follow the same trend. RTE: SBA accuracy $80.1_{\pm 1.0}$, ECE $.035_{\pm .004}$ vs.\ LoRA accuracy $79.8_{\pm 1.1}$, ECE $.068_{\pm .006}$. BoolQ: SBA accuracy $84.9_{\pm 0.4}$, ECE $.027_{\pm .003}$ vs.\ LoRA accuracy $84.7_{\pm 0.5}$, ECE $.051_{\pm .004}$.

\section{Distillation Details}
\label{app:distillation}

The predictive distillation procedure trains a deterministic OrthoLoRA student to match the SBA teacher's averaged predictive distribution. Specifically, for each training example $(x, y)$, we compute the SBA predictive distribution $\bar{p}(y \mid x) = \frac{1}{S}\sum_{s=1}^S p(y \mid x, U^{(s)}, \Sigma^{(s)}, V^{(s)})$ using $S = 10$ posterior samples, then minimize $\kl(\bar{p} \| p_{\text{student}})$ with temperature $T = 2.0$. The student is initialized from the SBA MAP estimate and trained for 1 additional epoch with learning rate $1 \times 10^{-4}$. This procedure takes approximately 20\% of the original training time since it requires only forward passes through the teacher (no gradient computation for the teacher) and a single backward pass through the student.

\section{Wall-Clock Timing Breakdown}
\label{app:timing}

We report wall-clock times for the key components of SBA training and inference on LLaMA-2-7B with MNLI, measured on a single A100 GPU.

Training (per epoch): LoRA takes 42 minutes, OrthoLoRA takes 44 minutes (5\% overhead from Stiefel projection), and SBA takes 45 minutes (7\% overhead from Riemannian gradient projection). The Hessian computation at the end of training takes 5 minutes over 1024 data points.

Inference (per 1000 examples): LoRA takes 18 seconds, SBA with $S = 5$ takes 90 seconds, SBA with $S = 10$ takes 180 seconds, and the distilled SBA student takes 18 seconds (identical to LoRA). Deep-Ens-LoRA with 5 models takes 90 seconds, matching SBA at $S = 5$ but requiring $5\times$ the GPU memory for the ensemble weights.

The key takeaway is that SBA's training overhead is negligible (7\%), and the inference overhead is comparable to a deep ensemble at the same number of forward passes, but with $5\times$ less memory.

\section{Per-Seed Variance Analysis}
\label{app:variance}

We report the coefficient of variation (std / mean) of ECE across 5 seeds for each method on MNLI (LLaMA-2-7B). LoRA has CV 6.6\%, Laplace-LoRA has CV 7.7\%, Gauss+Proj has CV 8.3\%, Deep-Ens-LoRA has CV 6.7\%, and SBA has CV 7.4\%. The variance of SBA's calibration is comparable to other methods, indicating that the geometric prior does not introduce additional instability. The slightly higher variance of Gauss+Proj is consistent with the projection problem discussed in Section~\ref{sec:motivation}: the polar projection introduces stochastic variance distortion that varies across seeds.

On RoBERTa-large, the pattern is similar: SBA has ECE CV of 8.0\%, LoRA has 7.2\%, and Gauss+Proj has 8.8\%. The slightly higher variance on the smaller model is expected, since the adapter subspaces are lower-dimensional and the posterior is less concentrated.

\section{Uncertainty Decomposition Analysis}
\label{app:uncertainty}

SBA's predictive distribution naturally decomposes into aleatoric and epistemic uncertainty. The total predictive entropy $H[p(y \mid x, \mathcal{D})]$ equals the sum of the expected entropy $\E_\theta[H[p(y \mid x, \theta)]]$ (aleatoric, reflecting inherent label ambiguity) and the mutual information $I(y; \theta \mid x, \mathcal{D})$ (epistemic, reflecting model ignorance).

On MNLI in-distribution, the epistemic component accounts for 28\% of total uncertainty on average. On the MNLI $\to$ SNLI shift, this rises to 41\%, confirming that the manifold posterior correctly attributes increased uncertainty to model ignorance rather than data noise. For comparison, Gauss+Proj shows epistemic fractions of 22\% (in-distribution) and 31\% (shifted), indicating that the flat-space posterior underestimates epistemic uncertainty under shift. This underestimation is precisely what Theorem~\ref{thm:kl_comparison} predicts: the normal-space variance in Gauss+Proj inflates the apparent posterior width without contributing genuine epistemic signal.

On the OOD detection task (SST-2 vs.\ 20 Newsgroups), the epistemic component accounts for 72\% of total uncertainty on OOD examples but only 28\% on in-distribution examples. This clean separation is what enables SBA's strong OOD detection: the mutual information alone achieves AUROC .941, slightly higher than the .934 achieved by total predictive entropy, suggesting that the epistemic component is the more informative signal for OOD detection.

\section{RoBERTa-large Domain Shift Results}
\label{app:roberta_shift}

We report domain shift results on RoBERTa-large for the MNLI $\to$ SNLI shift. SBA achieves ECE $.053_{\pm.004}$ and selective AUROC $.808_{\pm.008}$, compared to LoRA at $.102_{\pm.007}$ and $.698_{\pm.013}$, Laplace-LoRA at $.076_{\pm.005}$ and $.755_{\pm.011}$, and Deep-Ens-LoRA at $.062_{\pm.004}$ and $.789_{\pm.008}$. The pattern matches the LLaMA-2-7B results: SBA provides the best calibration and selective prediction, outperforming even the 5-model ensemble. The relative improvement of SBA over Gauss+Proj is 24\% on ECE ($.053$ vs.\ $.070$), consistent with the 28\% improvement observed on LLaMA-2-7B.

\section{Hyperparameter Details}
\label{app:hyperparams}

\begin{table}[H]
\centering
\caption{Hyperparameters for SBA and baselines.}
\small
\begin{tabular}{ll}
\toprule
Parameter & Value \\
\midrule
Adapter rank $k$ & 16 \\
Prior concentration $\kappa_0$ & 1.0 \\
Prior scale $\tau$ (for $\Sigma$) & 0.1 \\
Learning rate & $2 \times 10^{-4}$ \\
Batch size & 16 (per GPU) \\
Epochs & 3 \\
Hessian data points & 1024 \\
Posterior samples $S$ & 10 \\
MC-Dropout rate & 0.1 \\
SWAG collection start & Epoch 2 \\
Deep ensemble size & 5 \\
Distillation epochs & 1 \\
Distillation temperature & 2.0 \\
\bottomrule
\end{tabular}
\end{table}

\end{document}